\documentclass[accepted]{uai2026} % after acceptance, for a revised version; 
\usepackage{hyperref}% hyperlinks
\usepackage{url}            % simple URL typesetting
\usepackage{booktabs}       % professional-quality tables
\usepackage{amsfonts}       % blackboard math symbols
\usepackage{nicefrac}       % compact symbols for 1/2, etc.
\usepackage{microtype}      % microtypography
\usepackage{xcolor}         % colors
\usepackage{macro}

\DeclareMathOperator{\CHG}{CHG}

\usepackage[american]{babel}
\usepackage{natbib} % has a nice set of citation styles and commands
\usepackage{mathtools} % amsmath with fixes and additions
\usepackage{booktabs} % commands to create good-looking tables
\usepackage{tikz} % nice language for creating drawings and diagrams
\usetikzlibrary{positioning} % enable positioning syntax like [right=of ...]
\usepackage{algorithm}      % algorithm float environment
\usepackage{algpseudocode}  % pseudocode (algorithmicx) for Algorithm boxes
\usepackage{dblfloatfix}    % allow full-width figure* at page bottom; fix their ordering

\title{Relaxing Faithfulness with Intervention-Only Causal Discovery}

\author[1,2]{\href{mailto:<bijan@dartmouth.edu>?Subject=Relaxing Faithfulness with Intervention-Only Causal Discovery}{Bijan H. S. Mazaheri}{}}
\author[2, 3]{Jiaqi Zhang}
\author[2, 3]{Caroline Uhler}

\affil[1]{%
    Thayer School of Engineering\\
    Dartmouth College\\
    Hanover, New Hampshire, USA
}
\affil[2]{%
    Schmidt Center\\
    Broad Institute of MIT and Harvard\\
    Cambridge, Massachusetts, USA
}
\affil[3]{%
    Laboratory for Information and Decision Systems\\
    Massachusetts Institute of Technology\\
    Cambridge, Massachusetts, USA
  }
  
\begin{document}
\maketitle

\begin{abstract}
Causal discovery algorithms learn a network that describes the causal dependencies among random variables. A common workflow involves first utilizing conditional independence properties on observational data to determine partially directed causal relationships, then applying interventions to orient the unknown causal directions. A critical assumption for the first step is faithfulness: a requirement that causally linked variables exhibit statistical dependence. Many natural systems include buffering and stabilizing pathways that cancel out to achieve systemic robustness. This cancellation of pathways violates faithfulness, leading causal discovery algorithms to incorrectly remove causal dependencies. In this paper, we argue that hard interventions contain information about the presence/absence of causal linkage that is overlooked in the first stage of structure discovery. We show that a mild assumption --- called intervention-immediacy faithfulness --- that allows cancellations, is sufficient to nonparametrically identify causal structures with hard interventions. These results position interventions as the primary carriers of information about causal structure, which should take precedence over conditional independence testing. To flip the paradigm, we also specify equivalence classes when the identification criteria are not met due to limitations in the scope of interventions. 
% This broadens permutation-based causal discovery into a more general refinement of transitive closures of DAGs.
\end{abstract}

\section{INTRODUCTION}

\paragraph{Causal Discovery.} Structural Causal Models (SCMs), popularized by the works of \citet{pearl1998graphs, pearl2009causality}, graphically describe networks of causal dependencies. Causal discovery is the task of recovering the underlying causal structure from data, typically in the form of a \emph{directed acyclic graph} (DAG) or a representation of an equivalence class of DAGs, though other targets, such as ancestral graphs with latent variables or cyclic models, are also studied (see \citet{vowels_dya_2022, squires2023causal} for reviews). One approach to causal discovery involves a \emph{constraint-based} search guided by conditional independence (``CI-tests''), e.g., the PC-algorithm from \citet{spirtes2000causation}.
Constraint-based approaches use the observation that variables without a direct causal link can be made independent by conditioning on variables on intermediary causal paths, known as the causal Markov condition \citep{pearl2009causality}.

Structural causal models imply a data-generating process whereby each random variable is generated from its (causal) parents. These processes are sometimes specified using equations, leading to structural \emph{equation} models (SEMs). Under these settings, these equations may be restricted using ``parametric assumptions'' on their functional form and noise. The most popular such assumption is linearity with additive Gaussian noise, yielding a multivariate Gaussian distribution on the full system.

Causal discovery algorithms therefore fall under two categories: (1) algorithms that make use of parametric assumptions on the structural equations, and (2) \emph{non}-parametric algorithms that only make use of conditional independence and other broad distributional properties. While parametric settings, such as Linear Non-Gaussian Additive Noise Models (LinGAMs) \citep{shimizu_linear_2006}, can sometimes fully recover a \emph{single} DAG structure, most non-parametric algorithms for causal discovery can only narrow the structure down to what is known as a ``Markov equivalence class,'' uniquely specified by its conditional independence constraints \citep{spirtes2000causation}. Such a class is usually represented first by an undirected ``skeleton,'' denoting the \emph{presence} of causal relationships, which is then partially oriented into a completed partial DAG (CP-DAG) \citep{spirtes2000causation}.

\paragraph{Interventions.} To resolve a Markov equivalence class, nonparametric causal discovery algorithms make \emph{secondary} use of \emph{interventions}. An intervention (a ``perturbation'' in biology) is a change to the system that \emph{may} elicit a change in causally downstream variables, and \emph{never} elicits a change in any other variables. Such interventions may be ``soft’’, such as ``shift interventions’’\citep{sani_identification_2020}, which change a mean while retaining randomness and also retaining dependence on causally upstream variables). Real-world examples of soft interventions include CRISPR ``gene knockdowns’’ \citep{chang_crisprcas9_2016} that partially suppress expression. Interventions may also be ``hard’’, such as ``do interventions’’ \citep{pearl2009causality} which interrupt the data-generating process, force random variables to take a specific value, and dissociate their dependence from causally upstream variables. Real-world examples of these hard interventions are ``gene knock\emph{outs}’’ \citep{guan2010review}, which \emph{fully suppress} their expression.

\paragraph{Faithfulness.}
Under causal sufficiency (no unobserved confounding), a data-generating process given by an SCM will exhibit the causal Markov condition, but the converse is not necessarily true: A causal path between two variables does not necessarily require that they be statistically dependent \citep{ramsey2012adjacency}. 
``\textit{Faithfulness}'' is the assumption that causal links imply statistical dependencies, which gives a two-way correspondence between CI-tests and graphical properties known as ``d-separation'' \citep{pearl2009causality}. This notion can be relaxed for interventions \citep{chevalley2025deriving}, requiring only that directed paths exhibit dependence.

Violations of both types of faithfulness are often due to ``path cancellation,'' i.e., dependencies whose effects cancel each other out.
While often assumed away, it is important to understand that path cancellations occur by design in engineered and evolutionarily-driven systems. For example, many biochemical pathways are known to be buffered using ``compensating paths'' \citep{hartman2001principles}. Hence, while path cancellations may be unlikely in a randomly selected model for synthetic data, they are \emph{very likely} in practice due to the formation of control systems. 
\begin{example}\label{example:compensating_paths}
    Gene $A$ may directly affect both $B, C$, but $B$ may act as a compensating path to stabilize the expression of $C$ even when $A$'s expression is suppressed. As such $A \rightarrow C$ and $A \rightarrow B \rightarrow C$ cancel out so that changes in $A$ do not elicit changes in $C$.
\end{example}
Notice that Example~\ref{example:compensating_paths} exhibits path cancellation for both the conditional independence and intervention-based faithfulness definitions.

%To better understand the effects of these violations, it is helpful to decompose the graphical information in an SCM into \emph{skeletal (adjacency) information} and \emph{orientation information}. Causal adjacency corresponds to the existence of a direct causal connection, agnostic of the direction. For example, $A, B$ are causally adjacent if $A \rightarrow B$ or $B \rightarrow A$. Orientation then specifies the direction of these causal adjacencies. The primary consequence of faithfulness violations is the removal of causal adjacencies that should not be removed.

\subsection{Contributions}

In this paper, we show that hard interventions can nonparametrically identify causal structures without assumptions that prohibit path cancellation. This result is based on a new notion of faithfulness, ``intervention immediacy (II) faithfulness,'' which we contrast with conditional independence (CI) faithfulness. II faithfulness only requires interventional dependence in a node's ``most immediate effects'' --- effects with only a single path and no opportunity for path cancellation (e.g., $B$ for an intervention on $A$ in Example~\ref{example:compensating_paths}).

While most notions of faithfulness are developed for nonparametric algorithms,\footnote{Parametric models use more precise notions tuned to their settings, such as the matrix condition number.} analyzing these assumptions in finite-sample regimes requires a parametric framework to quantify ``closeness'' to a violation. CI faithfulness has been studied under the parametric assumption of a multivariate Gaussian, where relaxation into ``$\varepsilon$-strong (CI) faithfullness'' makes the assumption very strong. To compare II faithfulness and CI faithfulness, we provide a parametric study in the multivariate Gaussian setting and compare our results to those of \citet{uhler2013geometry}. II faithfulness is significantly milder than CI faithfulness under these metrics.

Our results position interventions as \emph{superior} to CI testing, 
% information-theoretically
not only for \emph{orienting} undirected relationships, but also for \emph{discovering} their presence. This challenges the prevailing hierarchy that positions conditional independence as the primary source of causal information for learning the ``causal skeleton'' and demotes interventions as a secondary refinement.

The disruption of the CI/intervention hierarchy suggests a flipped workflow: utilize information from the available interventions first, then refine the \emph{intervention} equivalence class with conditional independence testing that does not require assumptions beyond II faithfulness. To reorient algorithms around intervention-based causal discovery, we define these intervention equivalence classes based on limitations of the intervention cardinality. 

\subsection{Overview}
Our algorithm is built on definitions of change sets and II faithfulness, formalized in Section~\ref{sec: ii_faithfulness}. II faithfulness is briefly analyzed in the parametric setting of Gaussian random variables with linear dependencies in Section~\ref{sec: parametric ii faith} for the sake of comparison to \citet{uhler2013geometry}. None of the algorithms in this paper require parametric assumptions.

In Section~\ref{sec: alg} we present two new causal discovery algorithms built on change sets and the assumption of II faithfulness. In this setting, identifiability depends on the number of variables on which we can intervene simultaneously ($k$) relative to the maximum local (vertex) connectivity $\kappa_{\max}$ between non-adjacent pairs of the graph (i.e., the largest, over all non-adjacent ordered pairs, of the minimum number of vertices whose removal destroys all directed paths between them).

To specify a new ``intervention equivalence class,'' we define the ``$k$-robust transitive closure'' of $\G$. Informally, starting from $\G$ and following a topological order (so that acyclicity is preserved), we add an edge $V_i \to V_j$ whenever $V_i$ and $V_j$ cannot be disconnected by removing fewer than $k$ vertices; that is, whenever \emph{every} vertex cut separating them has size at least $k$. Under limited cardinality of interventions, DAGs that have the same $k$-robust transitive closure form an equivalence class, which we cannot distinguish between using $\leq k$-node interventions. We prove two guarantees: (1) a graph is identifiable with $\mathcal{O}(|\bvec{V}|^2)$ hard interventions on up to $\kappa_{\max}+1$ nodes, and (2) under a cardinality limit $k$ it is identifiable up to its $k$-robust transitive closure. These equivalence classes provide insight into the information redundancies between intervention outcomes and conditional independence testing, and may be further resolved using regular CI tests as secondary sources of information.

In Section~\ref{sec: empirical}, we perform an empirical study on synthetic data to verify the relative robustness of intervention-only causal discovery. In the absence of real datasets with large-scale interventions, our results motivate the development of new methods for performing multi-node interventions to improve the accuracy of causal discovery.

\subsection{Related Works}

\paragraph{Interventions in Causal Discovery.}
Seminal work by \citet{eberhardt2007interventions} studied the integration of interventions into causal discovery, particularly on orienting Markov equivalence classes.
\citet{shanmugam2015learning} explored limited intervention size when orienting a Markov equivalence class, but not learning the graph. \citet{hauser2012characterization} defined interventional Markov equivalence classes with respect to intervention sets, then extended this to joint observational and interventional measurements in \citet{hauser2015jointly}, which provided a critical framework for utilizing an intervention-second approach to causal discovery. \citet{yang2018characterizing} and \citet{kocaoglu_characterization_2019} extended this characterization to soft interventions.

A large body of modern work integrates observational and interventional data \emph{jointly} rather than sequentially, including GIES \citep{hauser2012characterization}, IGSP and its variants \citep{wang2017permutation, squires2020permutation}, JCI \citep{mooij_joint_2020}, and score-based continuous-optimization methods such as DCDI \citep{brouillard_differentiable_2020}. These methods typically operate \emph{offline}, pooling a fixed observational sample with a pre-collected pool of interventional data. Our approach is instead \emph{online} and \emph{active}, in the tradition of adaptive experimental design for causal discovery \citep{he_active_2008, shanmugam2015learning, choo2022verification}: because the pruning step dynamically searches for a minimum vertex cut $\bvec{Z}$, it designs targeted $k$-node interventions on the fly. It departs from that line in what the interventions are \emph{for}: prior active methods design interventions to \emph{orient} a skeleton that is still recovered from observational conditional-independence tests, whereas we use interventions to establish edge \emph{presence}, relying on observational CI only as an optional secondary refinement when the intervention cardinality is limited---never for the skeleton itself. This is what lets us side-step path cancellation entirely.

The work conceptually closest to ours is \citet{chevalley2025deriving}, who introduce \emph{intervention faithfulness} and use single-variable interventions to recover a causal order; II faithfulness weakens their condition so that it tolerates path cancellation (Lemma~\ref{lem: ii_subset_ci}, Theorem~\ref{thm: geometry of intervention faithfulness}), and our algorithms recover full structure rather than only an ordering. Also closely related, \citet{zhou_characterization_2025} characterizes and learns causal graphs from hard interventions in the \emph{presence of latents} via conditional-independence-invariance machinery---complementary to our causally-sufficient but cancellation-tolerant setting. Two further methods are contrasts rather than close relatives: DCDI (above) assumes standard ($\lambda$-strong) faithfulness, which path cancellation violates, and identifies only up to the interventional Markov equivalence class, whereas our framework is nonparametric and faithfulness-relaxed; and effect identification given a known partial ancestral graph \citep{jaber_causal_2022} is orthogonal, addressing identification rather than structure learning.

\paragraph{Finite Sample Concerns for Faithfulness.} \citet{meek_strong_1995} argued that faithfulness is a relatively mild assumption under exact statistics because it has full Lebesgue measure.
However, success under finite-sample uncertainty requires a stronger notion --- referred to as $\lambda$-strong faithfulness \citep{zhang2012strong} --- to ensure that stochastic deviations from faithfulness are unlikely. From this perspective, \citet{uhler2013geometry} showed that faithfulness is a strong assumption because the manifold of violations behaves like a space-filling curve. These results are discussed in more detail in \Cref{sec: prelims}.

\paragraph{Causal Direction and Structure Learning.}
Significant recent work has shown that access to a causal ordering, even if such an ordering is searched for, drastically improves the performance of causal discovery. A notable algorithm that searches along permutations is ``greedy sparsest permutation'' (GSP) \citep{wang2017permutation}. See \citet{squires2023causal} for a review. These approaches are still based on CI testing, but we believe they are the preliminary foundations of a new causal discovery workflow.

A related discovery involves improving the number of CI tests utilized by causal discovery by integrating edge-orientation sooner than the original PC algorithm \citep{shiragur2024causal,mones2026number}, which is still CI-based, gleaning orientations from unsheilded collider structures.

\paragraph{Limiting Testing.}
\citet{kocaoglu2024characterization} characterized the equivalence classes of graphs with conditional independence tests with restricted conditioning sets. We similarly characterize the equivalence classes of limited \emph{interventions}.

\paragraph{Information Theoretic Causal Discovery.}
Intervention-based discovery shares many algorithmic similarities with information-theoretic causal discovery \citep{janzing2010causal, xu2025informationtheoretic}. In particular, the algorithms in these approaches have a similar structure: they begin by isolating the transitive closure of the true graph and then removing edges. Our work diverges from this approach by focusing on interventions and the responses they elicit within the system.

\section{Preliminaries} \label{sec: prelims}

\subsection{Notation}
We will use the capital Roman alphabet to denote random variables (e.g., $A, B, C, V$) and the lowercase Roman alphabet to denote assignments to those random variables (e.g., $A=a$ or just $a$). Bold will indicate a set of random variables, e.g., $\bvec{V} = (V_1, V_2, \ldots)^\top$, and $\bvec{v}$ is an assignment to $\bvec{V}$. Parents ($\PA$), children ($\CH$), ancestors ($\AN$), and descendants ($\DE$) in graphs will also follow these conventions, e.g., $\PA(V)=\pa^{\bvec{v}}(v)$, where the assignments to those parents come from values specified in $\bvec{v}$. We use subscripts to indicate the relevant graph structure, e.g., $\PA_{\mathcal{G}}(V)$. 
We will generally use the Greek alphabet (e.g., $\alpha, \beta$) to represent parameters for structural equations and thresholds to quantify faithfulness.

Throughout the paper we assume the causal structure is a DAG and that the system is \emph{causally sufficient} (no unobserved confounders). We follow the convention that a vertex is \emph{not} counted among its own ancestors or descendants. The \emph{transitive closure} $\text{TC}(\G)$ of a DAG $\G$ is the graph on the same vertex set that contains a directed edge $V_i \to V_j$ whenever there is a directed path from $V_i$ to $V_j$ in $\G$ (equivalently, whenever $V_j \in \DE_{\G}(V_i)$). The \emph{local (vertex) connectivity} of a non-adjacent ordered pair $(V_i, V_j)$ is the minimum number of vertices whose removal destroys all directed paths from $V_i$ to $V_j$; by Menger's theorem this equals the maximum number of internally vertex-disjoint directed $V_i \rightsquigarrow V_j$ paths. We write $\kappa_{\max}$ for the maximum of this quantity over all non-adjacent ordered pairs.

\subsection{Interventions}
An intervention on a variable modifies its natural causal mechanism. In a general setting, multiple distinct interventions may exist for a single variable (e.g., setting its value, shifting its mean, etc.). For the purposes of this paper, however, we consider a simplified setting: for each variable $V \in \bvec{V}$, there exists a single, unique \textbf{hard intervention} that can be applied. We denote the simultaneous application of this unique intervention to all variables in a set $\bvec{A} \subseteq \bvec{V}$ as $I(\bvec{A})$. This operation severs the dependence between the variables in $\bvec{A}$ and their parents, causing their value to be determined by a new, exogenous mechanism. We denote the resulting probability distribution as $P_{I(\bvec{A})}$, where $P$ denotes the observational distribution (i.e., no intervention). All interventions in this paper are hard interventions in this sense. Where we informally speak of a ``strong'' intervention, we mean a hard intervention whose induced change is large enough (relative to the noise) to be detected by a finite-sample two-sample test; this signal magnitude is precisely what the parameter $\lambda$ quantifies in Section~\ref{sec: parametric ii faith}.

\subsection{Faithfulness}
D-separation provides a graphical criterion that constitutes a necessary but not sufficient condition for conditional independence under the causal Markov condition. We say that two variables $V_i, V_j$ are d-separated if there is no active path between them in the causal graph. See \citet{pearl2009causality} for a more detailed description.

To provide intuition about path cancellation, we give an example with linear Gaussian structural equation models (SEMs). CI faithfulness with respect to a causal graph $\G$ corresponds to nonzero covariance for all d-connected pairs. Consider the following example with $N_1, N_2, N_3$ as independent, zero-mean, unit-variance Gaussians. 
\begin{equation}\label{eq: model for faithfulness violation}
\begin{aligned}
X_1 &= N_1,\\
X_2 &= \alpha_{12} X_1 + N_2,\\
X_3 &= \alpha_{13} X_1 + \alpha_{23} X_2 + N_3.
\end{aligned}
\end{equation}
The covariance between $X_1$ and $X_3$ is a function of the model coefficients:
\begin{align*}
\Cov(X_1, X_3) &= \alpha_{13} + \alpha_{23} \alpha_{12}.
\end{align*}
A monomial of parameters emerges from each active path between $X_1$ and $X_3$. When $\alpha_{13} + \alpha_{23} \alpha_{12} = 0$, we have $\Cov(X_1, X_3) = 0$, a CI faithfulness violation. Such violations occur in two ways: (1) trivial coefficients (e.g., $\alpha_{ij}=0$) and (2) ``cancellation'' of monomial terms from multiple paths.

$\lambda$-strong faithfulness \citep{uhler2013geometry} requires  $\abs{\Cov(V_i, V_j)} > \lambda \sqrt{\Var(V_i) \Var(V_j)}$ for all d-connected $V_i, V_j$. Geometrically, violations correspond to distributions close to the hypersurfaces defined by regular CI faithfulness violations. The argument for $\lambda$-strong faithfulness stems from finite-sample uncertainty, where true and empirical covariance differ.

\begin{theorem}[informal, \citep{uhler2013geometry}]\label{thm: geometry of CI faithfulness}
The volume of $\lambda$-strong CI faithfulness violations in a linear SEM with coefficients in $[-1, 1]$ is at least $\omega(\lambda^{\text{poly}(n)} 2^{\abs{\bvec{E}}})$ in the worst case.
\end{theorem}

\subsection{Intervention Faithfulness}
\citet{chevalley2025deriving} introduced a notion of ``$\varepsilon$-strong intervention faithfulness'' that formalizes the requirement that intervening on a variable must induce a change on all of its descendants. Such an assumption is still vulnerable to a ``cancellation of paths.'' For example, in the model from Eq.~\eqref{eq: model for faithfulness violation}, when $\alpha_{13} + \alpha_{23}\alpha_{12} = 0$, changing $X_1$ will elicit no change in $X_3$, violating this assumption.

\section{Assumptions and Framework} \label{sec: ii_faithfulness}
To address the problem of path cancellation, we develop intervention immediacy faithfulness (II faithfulness) as a milder and more robust assumption for causal discovery. We also develop intervention locality as a complementary interventional analog to the causal Markov property.

\begin{definition}[Most Immediate Child (MIC)]\label{def: immediate child}
    We say that $B \in \bvec{V}$ is a ``most immediate child'' (MIC) of $A \in \bvec{V}$ in $\G = (\bvec{V}, \bvec{E})$ if $B \in \CH(A)$ and there are \emph{no other directed paths} from $A$ to $B$.
\end{definition}
Here ``no other directed paths'' refers to \emph{all} directed paths in $\G$ (of any length), not only paths of length one; equivalently, $B$ is a MIC of $A$ if and only if the edge $A \to B$ is the \emph{unique} directed path from $A$ to $B$.
\begin{lemma} \label{lem: always immediate child}
    In a DAG, for any vertex $V \in \bvec{V}$, if $\CH(V) \neq \emptyset$, then there is always at least one MIC.
\end{lemma}
To see the intuition behind Lemma~\ref{lem: always immediate child}, consider the fact that any child with more than one path from $V$ must have one path that is longer than length one, which means there is another child earlier in the topological ordering. Repeat this until you find the ``most immediate child.'' The proof is deferred to Appendix~\ref{apx: deferred proofs}.

II faithfulness is based on the notion of a ``change'' between different interventional distributions.
\begin{definition}[Change Set]
    The change set $\CHG(I(\bvec{A}))$ is the set of all variables whose marginal distributions change under an intervention $I(\bvec{A})$.
    \[
        \CHG(I(\bvec{A})) \coloneqq \{V \in \bvec{V} \quad \text{s.t.} \quad P_{I(\bvec{A})}(v) \neq P(v)\}.
    \]
\end{definition}
\begin{definition}[Conditional Change Set]
    Let $I(\bvec{A})$ and $I(\bvec{B})$ be interventions on disjoint sets $\bvec{A}$ and $\bvec{B}$. The conditional change set of $I(\bvec{A})$ given $I(\bvec{B})$ is the set of variables whose distributions change when adding the intervention $I(\bvec{A})$ to the existing intervention $I(\bvec{B})$.
    \[
        \CHG(I(\bvec{A}) \mid I(\bvec{B})) \coloneqq \{V \in \bvec{V} \mid P_{I(\bvec{A} \cup \bvec{B})}(v) \neq P_{I(\bvec{B})}(v)\}.
    \]
\end{definition}
Conditional change sets arise naturally in sequential experimentation: a canonical example is sequential gene knockouts, where a first knockout $I(\bvec{B})$ removes upstream regulatory effects so that the effect of a second knockout $I(\bvec{A})$ can be attributed to the paths that remain. This mechanism is what step~2 of our algorithms exploits.
A change set may include descendants of the intervened variables, as well as the intervened variables themselves. However, it cannot include non-descendants.
\begin{definition}[Intervention Locality]
    $I(\bvec{A})$ satisfies intervention locality if its change set is restricted to the intervened variables and their descendants in $\G$.
    \[
        \CHG(I(\bvec{A})) \subseteq \DE_{\G}(\bvec{A}) \cup \bvec{A}.
    \]
\end{definition}
Intervention locality is satisfied by the hard interventions considered throughout this paper (and, more generally, by any intervention that only modifies the mechanisms of the intervened variables); as noted below, it is implied by the SCM data-generating process. A critical property of hard interventions is the severance of relationships between the intervened variables and their causes. This breaks down potentially canceling paths, giving rise to the notion of ``conditional'' MICs. Throughout, a \emph{conditioning intervention} $I(\bvec{B})$ is a hard intervention on a disjoint set $\bvec{B}$ that is held fixed while we probe the effect of a second intervention.
\begin{definition}[Conditional MIC]
    Let $\G = (\bvec{V}, \bvec{E})$ be a graph and $I(\bvec{B})$ be a conditioning intervention on the set $\bvec{B} \subset \bvec{V}$. The residual graph $\G_{I(\bvec{B})}$ is the subgraph formed by removing the vertices in $\bvec{B}$ and their incident edges. We say $C$ is a ``conditional MIC'' of $A$ given $I(\bvec{B})$ if $C$ is a MIC of $A$ in the residual graph $\G_{I(\bvec{B})}$.
\end{definition}
\begin{definition}[II Faithfulness]\label{def: strong int faithfulness}
    A model is II faithful if for any pair of disjoint intervention sets $\bvec{A}$ and $\bvec{B}$, all conditional MICs of any $A \in \bvec{A}$ given $I(\bvec{B})$ are in the corresponding conditional change set $\CHG(I(\bvec{A}) \mid I(\bvec{B}))$.
\end{definition}
Intervention locality and II faithfulness will make up the two main assumptions of our intervention-only causal discovery framework. 
Like the causal Markov condition, intervention locality is implied by the SCM data-generating process: a hard intervention on $\bvec{A}$ replaces only the structural assignments of the variables in $\bvec{A}$ and leaves every other mechanism intact, so a variable's distribution can change only if it is reachable from $\bvec{A}$ along directed paths---that is, only if it lies in $\bvec{A} \cup \DE_{\G}(\bvec{A})$.
II faithfulness is similar to CI faithfulness in that it must be assumed. However, this assumption is significantly weaker, as we will explain by studying it in a parametric setting.

\section{Parametric II Faithfulness} \label{sec: parametric ii faith}
To analyze the parametric implications of II faithfulness, we return to the setting of linear SEMs where each variable is a linear function of its parents plus an independent noise term $\varepsilon_i$ with $\mathbb{E}[\varepsilon_i]=0$ and $\text{Var}(\varepsilon_i)=1$. We can quantify II faithfulness with a $\lambda$-strong version.
\begin{definition}[$\lambda$-strong Mean II faithfulness]\label{def: mean ii faithfulness}
    A linear SEM is ``$\lambda$-strong mean II faithful'' if the following holds for every variable $A$, every conditioning intervention $I(\bvec{B})$ with $A \notin \bvec{B}$, and every conditional MIC $C$ of $A$ given $I(\bvec{B})$: there exists an intervention on $A$ that shifts its mean by some non-zero amount $\delta$ such that the induced change in the mean of $C$ (relative to the distribution under $I(\bvec{B})$ alone) satisfies $|\E_{I(\{A\}) \mid I(\bvec{B})}[C] - \E_{I(\bvec{B})}[C]| \ge \lambda|\delta|$. In particular, the intervention (and hence $\delta$) may depend on $C$ and $\bvec{B}$; we do \emph{not} require a single intervention to work uniformly across all conditioning sets.
\end{definition}

\paragraph{Distributional vs.\ mean faithfulness.} Definition~\ref{def: strong int faithfulness} (II faithfulness) requires a change in the full \emph{distribution} of each conditional MIC, whereas Definition~\ref{def: mean ii faithfulness} ($\lambda$-strong mean II faithfulness) constrains the \emph{mean}. 
These are not equivalent, and mean II faithfulness is the strictly stronger requirement: an intervention can change a conditional MIC's distribution while leaving its mean fixed. Our parametric analysis and our algorithmic implementation (Section~\ref{sec: empirical}) target mean shifts, and thus certify only mean II faithfulness. Replacing the difference-in-means test with a nonparametric two-sample test (e.g., a Kolmogorov--Smirnov or maximum-mean-discrepancy test) recovers the full distributional notion of Definition~\ref{def: strong int faithfulness} empirically.

A similar notion can be defined for the variance, but we stick with means for simplicity, noting that this stronger assumption provides upper bounds on the ``strength'' of also including variance II faithfulness. $\lambda$-strong mean II faithfulness can be characterized parametrically in this setting.
\begin{lemma}\label{lem: all_coeffs_bounded}
    In a linear SEM with unit-variance noise, a model is $\lambda$-strong mean II faithful if and only if every path coefficient in the model is bounded away from zero, i.e., $|\beta_{ij}| \ge \lambda$ for all $(i, j) \in \bvec{E}$.
\end{lemma}
This proof is deferred to Appendix~\ref{apx: deferred proofs}. This lemma reveals that $\lambda$-strong mean II faithfulness is parametrically equivalent to preventing any single causal link from being trivially weak. In the example from Eq.~\eqref{eq: model for faithfulness violation}, this means $|\alpha_{12}| \ge \lambda$, $|\alpha_{23}| \ge \lambda$, and $|\alpha_{13}| \ge \lambda$.

\paragraph{A local characterization.} Lemma~\ref{lem: all_coeffs_bounded} is inherently \emph{local}: mean II faithfulness constrains only individual edge coefficients $\beta_{ij}$, and never sums or products of coefficients along competing paths (as in CI faithfulness). Equivalently, it suffices to require that the single coefficient $\beta_{AC}$ be bounded away from zero for each node $A$ and each conditional MIC $C$ of $A$. Every edge $(A, C) \in \bvec{E}$ is realized as such a pair, by conditioning on one internal vertex of every other $A \rightsquigarrow C$ path. Hence, this local condition coincides with the edge-wise bound $|\beta_{ij}| \ge \lambda$ for all $(i,j) \in \bvec{E}$.
\begin{lemma}\label{lem: ii_subset_ci}
    In a linear SEM with unit-variance noise, the set of parameterizations that violate $\lambda$-strong mean II faithfulness is a \emph{strict subset} of those that violate $\lambda$-strong CI faithfulness.
\end{lemma}
The intuition behind this proof comes from Lemma~\ref{lem: all_coeffs_bounded}, since trivial coefficients must also result in some trivial correlations (relative to the unit variance). The full proof is deferred to Appendix~\ref{apx: deferred proofs}.

Finally, we can characterize the volume of II faithfulness violations.
\begin{theorem}\label{thm: geometry of intervention faithfulness}
The volume of $\lambda$-strong mean II faithfulness violations within the parameter space where all edge coefficients $\beta_{ij}$, $(i,j)\in\bvec{E}$, lie in $[-1, 1]$, and for $\lambda \in (0,1)$, is $\mathcal{O}(\lambda |\bvec{E}|)$.
\end{theorem}

This linear growth in $\lambda$, when compared to Theorem~\ref{thm: geometry of CI faithfulness}, illustrates that II faithfulness is a much milder assumption than CI faithfulness.

\section{Algorithm}\label{sec: alg}
In Section~\ref{sec: ii_faithfulness}, we laid the foundation for intervention-only causal discovery using change sets and introduced two key assumptions: II faithfulness and intervention locality. We will now present an algorithm that uses this non-parametric framework to learn causal structures, identifying the conditions for complete recovery and establishing equivalence classes otherwise. For this entire section, $\G$ will denote the true graph with $n$ vertices, and we will assume the model satisfies II faithfulness and intervention locality.

The identifiability of causal discovery from interventions depends on the number of variables that can be simultaneously intervened upon. We refer to this as the cardinality of the intervention, denoted $|\bvec{A}|$ for an intervention $I(\bvec{A})$. We will first present an algorithm for the unrestricted case and then provide versions for limited intervention cardinality.

\subsection{Unrestricted Cardinality}
When the number of simultaneous interventions is not limited, the Unrestricted Intervention Cardinality (UIC) algorithm (Algorithm~\ref{alg: uic}) recovers the true graph $\G = (\bvec{V}, \bvec{E})$. Step 1 performs every single-node intervention and takes the transitive closure of the observed change sets to obtain a supergraph $\G'$ of $\G$; step 2 then walks $\G'$ in topological order and removes each edge $V_i \to V_j$ whose endpoint is not a conditional MIC of $V_i$ given the remaining parents.

\begin{algorithm}[t]
\caption{Unrestricted Intervention Cardinality (UIC)}
\label{alg: uic}
\begin{algorithmic}[1]
\Require change-set oracle $\CHG(\cdot)$ obtained from hard interventions; vertex set $\bvec{V}$.
\Ensure estimated causal DAG $\G''$.
\State $\G_{temp} \gets (\bvec{V}, \emptyset)$
\For{each $V_i \in \bvec{V}$} \Comment{Step 1: recover the transitive closure}
  \State perform $I(\{V_i\})$
  \For{each $V_j \in \CHG(I(\{V_i\})) \setminus \{V_i\}$}
    \State add edge $V_i \to V_j$ to $\G_{temp}$
  \EndFor
\EndFor
\State $\G' \gets \text{TC}(\G_{temp})$; \quad $\G'' \gets \G'$
\For{each $V_j$ in a topological order of $\G'$} \Comment{Step 2: prune spurious edges}
  \For{each $V_i \in \PA_{\G'}(V_j)$}
    \State $\bvec{B} \gets \PA_{\G''}(V_j) \setminus \{V_i\}$
    \If{$V_j \notin \CHG(I(\{V_i\}) \mid I(\bvec{B}))$}
      \State remove edge $V_i \to V_j$ from $\G''$
    \EndIf
  \EndFor
\EndFor
\State \Return $\G''$
\end{algorithmic}
\end{algorithm}

To prove correctness, we first characterize the graph $\G'$ learned in step 1.

\begin{lemma} \label{lem: step 1}
    The graph $\G'$ from step 1 is the transitive closure of the true graph $\G$.
\end{lemma}
The two inclusions are shown separately: $\text{TC}(\G_{temp}) \subseteq \text{TC}(\G)$ follows from intervention locality (every edge added to $\G_{temp}$ marks a true descendant relation), and $\text{TC}(\G) \subseteq \text{TC}(\G_{temp})$ by induction in reverse topological order, using II faithfulness to realize each true edge $V_m \to V_j$ as a path through a MIC of $V_m$. The full proof is given in Appendix~\ref{apx: deferred proofs}.

We now show that step 2 correctly prunes all spurious edges from the transitive closure without removing any true edges.

\begin{lemma}\label{lem: step 2}
    An edge $V_i \to V_j$ is in the true graph $\bvec{E}$ if and only if $V_j \in \CHG(I(\{V_i\}) \mid I(\PA_{\G''}(V_j) \setminus \{V_i\}))$.
\end{lemma}
This proof relies on the fact that a superset of the parent set of $V_j$ is easily obtained from the transitive closure that we recovered in the previous step, which can block all paths from $V_i$ to $V_j$. The full proof is given in Appendix~\ref{apx: deferred proofs}.

\begin{theorem}\label{thm: uic}
Assuming II faithfulness and intervention locality, the UIC algorithm learns $\G$ using $\mathcal{O}(n^2)$ interventions of cardinality up to $n-1$.
\end{theorem}
\begin{proof}
    Step 1 performs $n$ interventions. Step 2 performs at most one intervention for each edge in the transitive closure $\G'$, which is at most $\mathcal{O}(n^2)$. The cardinality of interventions in step 2 can be up to $n-1$. Correctness follows from Lemmas \ref{lem: step 1} and \ref{lem: step 2}.
\end{proof}

\subsection{Restricted Cardinality}
The UIC algorithm may require interventions on $n-1$ nodes simultaneously, which is often unrealistic. We now consider the case where the cardinality of any intervention is limited, i.e., $|\bvec{A}| \le k$. The following $k$-Restricted Intervention Cardinality ($k$-RIC) algorithm shows that identifiability is still possible if $k$ is larger than the maximum local (vertex) connectivity $\kappa_{\max}$ between non-adjacent vertex pairs, i.e., the largest number of vertices needed to block all directed paths between any non-adjacent pair. Throughout, a set $\bvec{Z}$ ``separates'' an ordered pair $(V_i, V_j)$ if removing $\bvec{Z}$ (as a hard intervention does to its targets' incoming edges) leaves no directed path from $V_i$ to $V_j$. Step 1 is identical to UIC; step 2 (Algorithm~\ref{alg: kric}) differs only in that, before testing an edge $V_i \to V_j$, it finds a \emph{minimum vertex cut} $\bvec{Z}$ separating $V_i$ from $V_j$ in $\G''$ and conditions on $I(\bvec{Z})$, pruning the edge when $|\bvec{Z}| < k$ and $V_j \notin \CHG(I(\{V_i\}) \mid I(\bvec{Z}))$.

\begin{algorithm}[t]
\caption{$k$-Restricted Intervention Cardinality ($k$-RIC)}
\label{alg: kric}
\begin{algorithmic}[1]
\Require change-set oracle $\CHG(\cdot)$; vertex set $\bvec{V}$; cardinality budget $k$.
\Ensure estimated graph $\G''$ (equal to $\G$ when $k > \kappa_{\max}$).
\State run Step 1 of Algorithm~\ref{alg: uic} to obtain $\G'$; \quad $\G'' \gets \G'$
\For{each $V_j$ in a topological order of $\G'$}
  \For{each $V_i \in \PA_{\G'}(V_j)$}
    \State $\G^- \gets \G''$ with edge $V_i \to V_j$ removed
    \State $\bvec{Z} \gets$ minimum vertex cut in $\G^-$ separating $V_i$ from $V_j$
    \If{$|\bvec{Z}| < k$ \textbf{ and } $V_j \notin \CHG(I(\{V_i\}) \mid I(\bvec{Z}))$}
      \State remove edge $V_i \to V_j$ from $\G''$
    \EndIf
  \EndFor
\EndFor
\State \Return $\G''$
\end{algorithmic}
\end{algorithm}

\begin{theorem}\label{thm: mic}
    Assume II faithfulness and intervention locality, and let $\kappa_{\max}$ denote the maximum, over all non-adjacent ordered pairs $(V_i, V_j)$, of the local connectivity of the pair (the minimum number of vertices whose removal destroys all directed paths from $V_i$ to $V_j$ in $\G$). Then the $k$-RIC algorithm with $k > \kappa_{\max}$ learns $\G$ using $\mathcal{O}(n^2)$ interventions.
\end{theorem}
The proof hinges on the topological order traversal, which guarantees that $\G''$ has the correct structure on the ancestors of vertex $V_i$ when $V_j$ is considered. Hence, we can find a vertex cut separating paths to $V_j$ that matches the minimum vertex cut in the true graph $\G$.

The full proof (by case analysis on true edges and non-edges, with the non-edge case an induction in topological order) is given in Appendix~\ref{apx: deferred proofs}.

\subsection{Partial Identification}
When the intervention cardinality $k$ is less than or equal to the local connectivity required to separate a non-adjacent pair, we lose identifiability. For example, consider two graphs: (1) $A \to B \to C$ and (2) $A \to B \to C, A \to C$. With only single-node interventions ($k=1$), we observe that $B, C \in \CHG(I(\{A\}))$, and $C \in \CHG(I(\{B\}))$. This information produces the same transitive closure for both graphs, and we cannot distinguish them without a cardinality 2 intervention ($I(\{B\})$ compared to $I(\{A, B\})$) or a CI test ($A \indep C \given B$).

\begin{definition}
    The $k$-robust transitive closure of a graph $\G$ is a supergraph of $\G$ containing an edge $V_i \to V_j$ (for $i, j$ in topological order) if the minimum vertex cut separating $V_i$ and $V_j$ in $\G$ has cardinality $k$ or greater.
\end{definition}

\begin{theorem}\label{thm: k-robust transitive closure}
    Assuming II faithfulness, the $k$-RIC algorithm recovers $\G$ up to its $k$-robust transitive closure.
\end{theorem}
\begin{proof}
    The algorithm prunes a spurious edge $V_i \to V_j$ if and only if it can find a separating vertex cut $\bvec{Z}$ of size less than $k$. If the minimum separating set for a non-edge has size $\ge k$, the spurious edge remains.
\end{proof}

\section{EMPIRICAL VALIDATION} \label{sec: empirical}

\paragraph{Setup.}
We instantiate intervention-based causal discovery with a two-sample Student-$t$ test (\verb|scipy.stats|, significance level $0.05$, Bonferroni-corrected for the number of tests) to detect a change in mean, and run the $k$-RIC algorithm for cardinality budgets $k \in \{1, 2, 3, 10\}$.\footnote{Code to reproduce all experiments is available on \href{https://github.com/honeybijan/Intervention-Only-Causal-Discovery}{GitHub}.} On $10$-node graphs, $k=10$ is unrestricted and coincides with UIC. We compare against two conditional-independence-reliant baselines. The first is the PC algorithm~\citep{spirtes2000causation} (implementation of \citealp{zhang2021gcastle}) run on the observational data initially, with the skeleton edges then oriented by the interventions in the usual way (``PC$\to$Interv''). The second is IGSP~\citep{wang2017permutation}, a joint interventional method, given the \emph{same} single-node interventional datasets that $k$-RIC uses in its first step; unlike $k$-RIC, its permutation search still relies on observational CI tests.

We evaluate on 30 random DAGs on 10 vertices with 25 expected edges, with additive standard Gaussian noise; the edge-weight range is regime-specific and stated below. To emulate a gene knockout, an intervention severs a node from its parents and resets it to a draw from a Gaussian centered at $\delta$ with standard deviation $0.5$---a hard (perfect) intervention producing a mean shift of $\delta$. Because the exogenous noise has unit variance, $\delta$ is expressed in units of the noise standard deviation (e.g., $\delta = 5$ is a five-standard-deviation shift). Keeping a non-negligible intervention variance leaves the intervened variable well-conditioned for the Gaussian CI and invariance tests that IGSP relies on, since a near-deterministic intervention would destabilize the partial correlations it estimates. We draw $1000$ observational samples and $1000$ samples per intervention. All methods receive the same observational sample and the same single-node interventional datasets; $k$-RIC's pruning is additionally allowed to query targeted interventions on the separating sets it searches over (adaptive, and used only for $k \geq 2$), which is intrinsic to the method, whereas PC$\to$Interv and IGSP are non-adaptive baselines that consume only the shared pool. We report the structural Hamming distance (SHD) between the true and recovered structures.

\paragraph{CI refinement.}
$k$-RIC returns the $k$-robust transitive closure, a supergraph of $\G$. We utilize a refinement step that removes the remaining spurious edges with observational CI tests: fixing a topological order of the recovered graph, we test each surviving edge $V_i \to V_j$ by conditioning on all other predecessors of $V_j$ in that order, deleting the edge when $V_i \indep V_j$ given that set (a partial-correlation test at level $\alpha$). Conditioning on the full predecessor set blocks every directed path into $V_j$, so a true edge's partial correlation stays proportional to its direct coefficient and is not eliminated by marginal cancellation; an edge is dropped only when that coefficient is too weak for the observational test to detect---the low-signal regime where interventions hold no advantage either. We apply this step in the coefficient-strength sweep below and the appendix analysis; the violin comparisons report pure $k$-RIC without this refinement, to isolate the cardinality effect.

\begin{figure*}[!tb]
    \centering
    \includegraphics[width=0.85\linewidth]{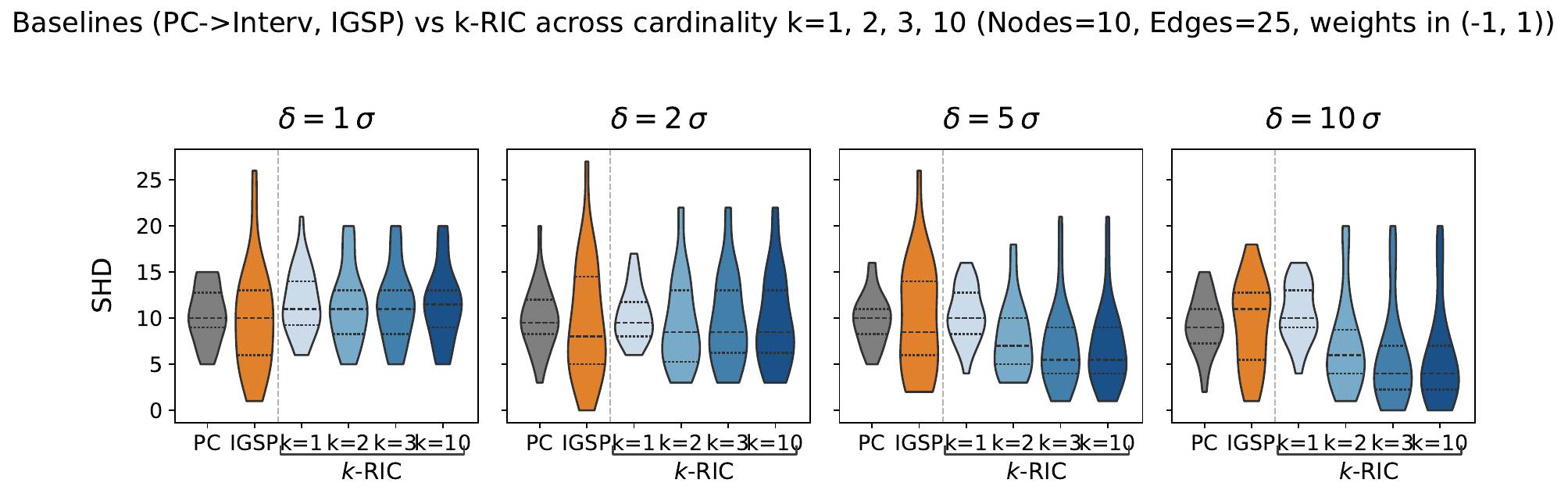}
    \caption{Intervention-only $k$-RIC across the cardinality budget $k \in \{1,2,3,10\}$ versus two CI-reliant baselines, at several intervention strengths $\delta$ (weights in $[-1,1]$; $30$ DAGs, $10$ nodes, $25$ edges; SHD, lower is better). \textbf{PC$\to$Interv} (grey) and \textbf{IGSP} (orange) are independent of $k$; the remaining violins are $k$-RIC at each $k$ (on $10$ nodes, $k=10$ is unrestricted). As $\delta$ grows, multi-node interventions ($k>1$) prune the spurious ancestral edges that single-node interventions ($k=1$) leave behind, and $k$-RIC drops below both baselines.}
    \label{fig:three algs}
\end{figure*}

\paragraph{Cardinality (standard regime, weights $[-1,1]$).}
We first isolate the effect of the intervention cardinality $k$, the quantity our theory turns on. At $k=1$, $k$-RIC performs no multi-node pruning and carries the spurious ancestral edges of the transitive closure. As the budget grows, multi-node interventions prune these edges and SHD falls---but the size of this effect is governed by $\delta$. When $\delta = 1\sigma$ the induced shifts are too small for the pruning tests to fire reliably, and all methods sit at comparable SHD; as $\delta$ increases, $k$-RIC improves monotonically with $k$ and drops clearly below both baselines (by $k=2$ for $\delta \geq 5\sigma$, and by $k=3$ at $\delta = 2\sigma$), reaching the lowest error of any method at $k=10$ (SHD $\approx 4$ at $\delta = 10\sigma$; Figure~\ref{fig:three algs}). The two baselines are independent of $k$, and both are overtaken once the interventions carry an appreciable signal. That an intervention-only method, using no CI tests, matches or beats methods that additionally exploit CI is direct evidence of the informational primacy of interventions; the advantage becomes overwhelming under path cancellation, where the CI baselines collapse (SHD $\approx 15$--$19$) while $k$-RIC recovers most of the structure (Appendix~\ref{apx: cancellation}, Figure~\ref{fig:lots of path cancellations}).

\begin{figure*}[!tb]
    \centering
    \includegraphics[width=0.8\linewidth]{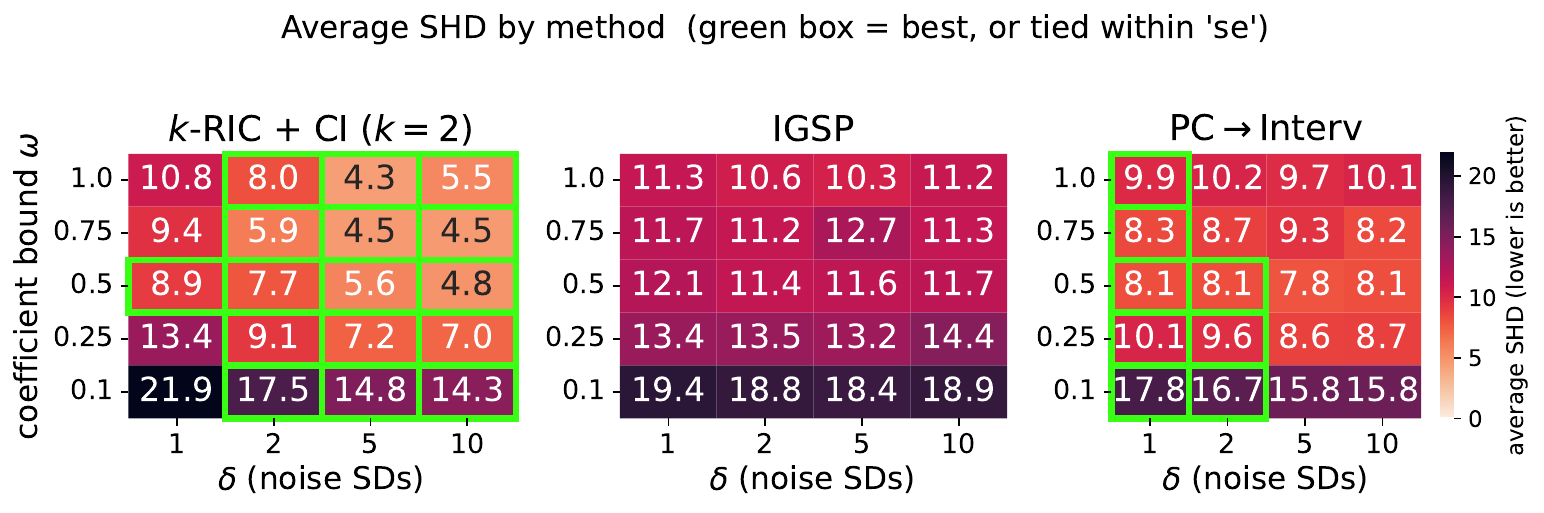}
    \caption{Average SHD (lower is better) of three methods across the coefficient bound $\omega$ and intervention strength $\delta$ (in noise SDs), on $10$-node graphs with $25$ edges ($20$ DAGs per cell). \textbf{Left:} the intervention-first pipeline ($k$-RIC at $k=2$, then a CI-refinement step). \textbf{Middle:} IGSP. \textbf{Right:} PC$\to$Interv. A green box marks each cell where a method is best, or within one standard error of the best. The intervention-first pipeline is best across most of the plane, ceding only the weakest-signal corner---the smallest $\delta$ or $\omega$---to PC$\to$Interv.}
    \label{fig:heatmap}
\end{figure*}

\paragraph{Coefficient-strength sweep (weights $[-\omega, \omega]$).}
Figure~\ref{fig:heatmap} compares the three methods across the coefficient range $[-\omega, \omega]$ and the intervention strength $\delta$, reporting average SHD over 20 DAGs per cell. The intervention-first pipeline---$k$-RIC at $k=2$ followed by a CI-refinement step---is best across the great majority of the plane, and its margin widens with the signal: for $\omega \geq 0.5$ and $\delta \geq 5$ it drives SHD down to roughly $4$--$6$, about half of either baseline. The conditional-independence baselines (PC$\to$Interv and IGSP) trail the pipeline across the bulk of the plane and catch up only in the weakest-signal corner---the smallest $\delta$ and smallest $\omega$---where the induced mean shift ($\propto \omega\delta$) is too small for the interventions to resolve structure and observational CI is all that remains. The two baselines degrade differently as $\omega$ weakens---PC$\to$Interv stays nearly flat (it orients from the interventions) while IGSP falls behind (its search leans on observational CI throughout)---as detailed in Appendix~\ref{apx: omega lineplot}.

\section{Discussion}

Differing geometries of faithfulness assumptions drive different robustness to finite-sample noise. We show that causal structure remains recoverable under a much milder notion of intervention faithfulness, positioning interventions as primary rather than supporting carriers of causal information. Empirically, intervention-only $k$-RIC outperforms the CI-based baselines PC and IGSP, and a secondary CI-refinement step yields a complete pipeline that returns a fully resolved DAG. Since a permutation is the weakest ($0$-robust) intervention equivalence class, generalizing permutation-based methods to $k$-robust transitive closures is a natural next step toward the intervention-first workflow.

\paragraph{Limitations.} Our framework does not address cancellation due to latent confounders, so extending II faithfulness to ancestral graphs (ADMGs/MAGs) is important future work. Our approach also presumes that hard interventions are available. This is natural in domains such as cell biology (gene knockouts), but interventions are infeasible or ethically impermissible in many settings (e.g., much of environmental and social science), where CI-based methods remain the appropriate tool. Finally, the algorithms as implemented test for mean shifts and therefore certify only mean II faithfulness; a nonparametric two-sample test recovers the distributional notion. Soft interventions would require strengthening II faithfulness, since they do not sever a node from its parents and thus do not break canceling paths.

\begin{acknowledgements} % will be removed in pdf for initial submission,
						 % (without ‘accepted’ option in \documentclass)
                         % so you can already fill it to test with the
                         % ‘accepted’ class option
    We thank the anonymous reviewers for their helpful feedbacks.
    B.M.~is partially supported by the Advanced Research Concepts (ARC) COMPASS program, sponsored by the Defense Advanced Research Projects Agency (DARPA) under agreement number HR001-25-3-0212. J.Z.~was partially supported  by the Eric and Wendy Schmidt Center. C.U.~was partially supported by NCCIH/NIH (1DP2AT012345), ONR (N00014-24-1-2687), DOE-ASCR (DE-SC0023187), and the Eric and Wendy Schmidt Center at the Broad Institute.
\end{acknowledgements}

% References
\bibliography{biblio}

\newpage

\onecolumn

\title{Relaxing Faithfulness with Intervention-Only Causal Discovery \\(Supplementary Material)}
\maketitle

\appendix
\setcounter{figure}{0}
\renewcommand{\thefigure}{A\arabic{figure}}

\section{Deferred Proofs} \label{apx: deferred proofs}

\subsection{Proof of Lemma~\ref{lem: always immediate child}}
\label{app:lem_always_immediate_child}

For any vertex $V \in \bvec{V}$, if $\CH(V) \neq \emptyset$, then there is always at least one MIC.

We assume $\G$ is a Directed Acyclic Graph (DAG). As $\G$ is a DAG, there exists a topological ordering of its vertices. Let's denote the position of a vertex $X$ in such an ordering by $\text{ord}(X)$. A key property of a topological ordering is that if there is a directed path from $X$ to $Y$, then $\text{ord}(X) < \text{ord}(Y)$.

Let $V$ be a vertex such that its set of children, $\CH(V)$, is non-empty. Consider the set of topological ordering positions for all children of $V$:
$$ S = \{ \text{ord}(C) \mid C \in \CH(V) \} $$
Since $\CH(V)$ is a finite, non-empty set, $S$ is a finite, non-empty set of integers. Therefore, $S$ must contain a minimum element.

Let $B \in \CH(V)$ be a child of $V$ that achieves this minimum, i.e., $\text{ord}(B) = \min(S)$. We claim that $B$ is a most immediate child of $V$.

Suppose for the sake of contradiction that $B$ is \emph{not} a most immediate child. By definition, this means there must exist another directed path from $V$ to $B$ besides the direct edge $V \to B$. Any such path must be of length at least 2 and must pass through another child of $V$. Let this path be:
$$ V \to C \to \dots \to B $$
where $C \in \CH(V)$ and $C \neq B$.

The existence of a directed path from $C$ to $B$ implies that $\text{ord}(C) < \text{ord}(B)$. However, $C$ is a child of $V$, so $\text{ord}(C) \in S$. We have found an element in $S$ that is strictly smaller than $\min(S)$, which is a contradiction.

Therefore, our assumption that $B$ is not a most immediate child must be false. It follows that $B$ is a most immediate child, and such a child must always exist if $\CH(V)$ is non-empty.

\subsection{Proof of Lemma~\ref{lem: all_coeffs_bounded}}
\label{app:lem_all_coeffs_bounded}

In a linear SEM with unit-variance noise, a model is $\lambda$-strong mean II faithful if and only if every path coefficient in the model is bounded away from zero, i.e., $|\beta_{ij}| \ge \lambda$ for all $(i, j) \in \bvec{E}$.

\textbf{($\Rightarrow$)} Assume the model is $\lambda$-strong mean II faithful. Consider an arbitrary edge $(A, C) \in \bvec{E}$ with coefficient $\beta_{AC}$. Let $\mathcal{P}_{AC}$ be the set of all other directed paths from $A$ to $C$. For each path $p \in \mathcal{P}_{AC}$, select one intermediate vertex $V_p$ along that path. Let the set of all such selected vertices be $\bvec{B} = \{V_p \mid p \in \mathcal{P}_{AC}\}$.

Now consider the intervention $I(\bvec{B})$. In the residual graph $\G_{I(\bvec{B})}$, all alternative directed paths from $A$ to $C$ are removed, leaving only the direct edge. Therefore, $C$ is a conditional MIC of $A$ given $I(\bvec{B})$. By the definition of $\lambda$-strong mean II faithfulness, an intervention on $A$ must produce a detectable change on $C$. In a linear model, this change is governed by the total effect in the residual graph, which is simply $\beta_{AC}$. The faithfulness condition thus requires $|\beta_{AC}| \ge \lambda$. Since $(A, C)$ was an arbitrary edge, this holds for all edges in $\bvec{E}$.

\textbf{($\Leftarrow$)} Assume $|\beta_{ij}| \ge \lambda$ for all edges $(i, j) \in \bvec{E}$. Consider any disjoint intervention sets $I(\bvec{A})$ and $I(\bvec{B})$, and let $C$ be a conditional MIC of some $A \in \bvec{A}$ given $I(\bvec{B})$. By definition, this means that in the residual graph $\G_{I(\bvec{B})}$, the only directed path from $A$ to $C$ is the direct edge $(A, C)$.

In a linear model, the change in the mean of $C$ from an intervention on $A$ (that shifts its mean by $\delta$) is determined by the total effect of $A$ on $C$ in the active graph, multiplied by $\delta$. Here, the total effect is simply the path coefficient of the single connecting path, $\beta_{AC}$. The change in mean is thus $\beta_{AC}\delta$. The condition for $\lambda$-strong mean II faithfulness requires $|\beta_{AC}\delta| \ge \lambda|\delta|$, which simplifies to $|\beta_{AC}| \ge \lambda$. By our initial assumption, this condition is met. Since this holds for any valid choice of $A, C, \bvec{B}$, the model is $\lambda$-strong mean II faithful.

\subsection{Proof of Lemma~\ref{lem: ii_subset_ci}}
\label{app:lem_ii_subset_ci}

In a linear SEM with unit-variance noise, the set of parameterizations that violate $\lambda$-strong mean II faithfulness is a strict subset of those that violate $\lambda$-strong CI faithfulness.

By Lemma~\ref{lem: all_coeffs_bounded}, a violation of $\lambda$-strong mean II faithfulness is equivalent to $|\beta_{ij}| < \lambda$ for some edge $(i,j) \in \bvec{E}$. Consider the adjacent pair $(V_i, V_j)$ and condition on $\bvec{S} = \PA_{\G}(V_j) \setminus \{V_i\}$. In a linear Gaussian SEM, the population regression coefficient of $V_i$ in the regression of $V_j$ on $\{V_i\} \cup \bvec{S}$ equals the structural coefficient $\beta_{ij}$ (since $V_j = \sum_{P \in \PA_\G(V_j)} \beta_{Pj} P + \varepsilon_j$ with $\varepsilon_j$ independent of the parents), and the partial correlation between $V_i$ and $V_j$ given $\bvec{S}$ is
\[
\rho(V_i, V_j \mid \bvec{S}) = \beta_{ij}\,\sqrt{\frac{\Var(V_i \mid \bvec{S})}{\Var(V_j \mid \bvec{S})}}.
\]
Because all conditional variances are finite and strictly positive in a non-degenerate model, the factor multiplying $\beta_{ij}$ is bounded and nonzero; hence $|\rho(V_i, V_j \mid \bvec{S})|$ is proportional to $|\beta_{ij}|$ and is likewise small whenever $|\beta_{ij}|$ is small. Since $V_i$ and $V_j$ are adjacent (hence d-connected given $\bvec{S}$), a sufficiently small $|\beta_{ij}|$ drives this partial correlation below the corresponding $\lambda$-strong CI faithfulness threshold, i.e., a $\lambda$-strong CI faithfulness violation. Thus every $\lambda$-strong mean II faithfulness violation is also a $\lambda$-strong CI faithfulness violation. Crucially, the converse fails: a model can violate CI faithfulness through cancellation (e.g., $\alpha_{13} + \alpha_{12}\alpha_{23} = 0$ in Eq.~\eqref{eq: model for faithfulness violation}) while every edge coefficient remains large, hence satisfying mean II faithfulness. Therefore the set of mean-II-unfaithful models is \emph{strictly} contained in the set of CI-unfaithful models.

\subsection{Proof of Theorem~\ref{thm: geometry of intervention faithfulness}}

\begin{proof}
Let the parameter space be the hypercube $[-1, 1]^{\abs{\bvec{E}}}$. By Lemma \ref{lem: all_coeffs_bounded}, a $\lambda$-strong mean II faithfulness violation occurs if and only if $|\beta_{ij}| < \lambda$ for at least one edge $(i, j) \in \bvec{E}$. For a single coefficient, the volume of violating parameters is the region where it lies in $(-\lambda, \lambda)$, which has a relative volume of $\lambda$. By applying a union bound over all $|\bvec{E}|$ edges, the relative volume of the violation set is at most $\mathcal{O}(\lambda |\bvec{E}|)$.
\end{proof}

\subsection{Proof of Lemma~\ref{lem: step 1}}

The graph $\G'$ from step 1 is the transitive closure of the true graph $\G$.

\begin{proof}
    Let $\G_{temp}$ be the graph constructed from the single-node interventions before taking the transitive closure. We need to prove that $\text{TC}(\G_{temp}) = \text{TC}(\G)$.

    First, we show $\text{TC}(\G_{temp}) \subseteq \text{TC}(\G)$. An edge $V_i \to V_j$ is added to $\G_{temp}$ only if $V_j \in \CHG(I(\{V_i\}))$. By intervention locality, this implies $V_j$ is a descendant of $V_i$ in $\G$. Therefore, every edge in $\G_{temp}$ corresponds to a path in $\G$, which means $\G_{temp} \subseteq \text{TC}(\G)$. The transitive closure of a subgraph is a subgraph of the original's transitive closure, so $\text{TC}(\G_{temp}) \subseteq \text{TC}(\G)$.

    Second, we show $\text{TC}(\G) \subseteq \text{TC}(\G_{temp})$. This is equivalent to showing that $\G \subseteq \text{TC}(\G_{temp})$. We proceed by induction over the source vertices in reverse topological order, showing that for the vertex currently under consideration, each of its outgoing edges $V_i \to V_j$ in $\G$ is realized as a directed path from $V_i$ to $V_j$ in $\G_{temp}$.

    Let the vertices be ordered $V_1, \ldots, V_n$ according to a reverse topological sort of $\G$.

    \textbf{Base case:} For the first vertex $V_1$ (a sink in $\G$), the claim is vacuously true as it has no outgoing edges.

    \textbf{Inductive step:} Assume for all $\ell < m$, the claim holds for $V_\ell$. Now consider $V_m$. Let $V_j$ be any child of $V_m$ in $\G$.
    If $V_j$ is a MIC of $V_m$, then by II faithfulness, $V_j \in \CHG(I(\{V_m\}))$, so the edge $V_m \to V_j$ is in $\G_{temp}$.
    If $V_j$ is not a MIC of $V_m$, then there exists another path from $V_m$ to $V_j$ in $\G$, which must pass through another child of $V_m$. Let this path be $V_m \to U_1 \to \dots \to V_j$.
    Since $U_1$ appears after $V_m$ in the reverse topological sort, our inductive hypothesis applies to it and all subsequent nodes on the path to $V_j$. This means there is a path $U_1 \leadsto V_j$ in $\G_{temp}$.
    This reduces our problem to showing there is a path $V_m \leadsto U_1$ in $\G_{temp}$. This is the same problem we started with for $V_j$. However, we can construct a sequence of children $V_j, U_1, U_2, \dots$ in which each child is required for a non-MIC path from $V_m$ to the preceding child. Since the number of children is finite, this sequence must terminate, and it must terminate at a child that is a MIC of $V_m$. Let this MIC be $M$. By II faithfulness, the edge $V_m \to M$ is in $\G_{temp}$. The path from $M$ to $V_j$ exists in $\G$, and all nodes on it are covered by the inductive hypothesis, so there is a path $M \leadsto V_j$ in $\G_{temp}$. Concatenating these, we get a path $V_m \to M \leadsto V_j$ in $\G_{temp}$.
    Thus, for any edge $V_m \to V_j$ in $\G$, there is a path in $\G_{temp}$. This means $\G \subseteq \text{TC}(\G_{temp})$, which proves $\text{TC}(\G) \subseteq \text{TC}(\G_{temp})$.
\end{proof}

\subsection{Proof of Lemma~\ref{lem: step 2}}
\label{app:lem_step_2}

An edge $V_i \to V_j$ is in the true graph $\bvec{E}$ if and only if $V_j \in \CHG(I(\{V_i\}) \mid I(\PA_{\G''}(V_j) \setminus \{V_i\}))$.

\textbf{($\Rightarrow$)} Assume $V_i \to V_j$ is in $\bvec{E}$. Let $\bvec{B} = \PA_{\G''}(V_j) \setminus \{V_i\}$. When we intervene on $\bvec{B}$, all paths into $V_j$ from its parents (other than $V_i$) are blocked. Because $\G'$ is the transitive closure, $\PA_{\G'}(V_j)$ contains all ancestors of $V_j$, so intervening on $\bvec{B}$ blocks all paths from $V_i$ to $V_j$ that do not begin with the edge $V_i \to V_j$. Thus, $V_j$ becomes a conditional MIC of $V_i$. By II faithfulness, $V_j \in \CHG(I(\{V_i\}) \mid I(\bvec{B}))$, so the edge is correctly kept.

\textbf{($\Leftarrow$)} Assume $V_i \to V_j$ is not in $\bvec{E}$. Let $\bvec{B} = \PA_{\G''}(V_j) \setminus \{V_i\}$. Since $V_i$ is not a true parent of $V_j$, the set of true parents $\PA_{\G}(V_j)$ is a subset of $\bvec{B}$. Any directed path from $V_i$ to $V_j$ in $\G$ must pass through at least one node in $\PA_{\G}(V_j)$. By intervening on $\bvec{B}$, all such paths are blocked. In the residual graph, $V_j$ is not a descendant of $V_i$. By intervention locality, $V_j \notin \CHG(I(\{V_i\}) \mid I(\bvec{B}))$, so the edge is correctly removed.

\subsection{Proof of Theorem~\ref{thm: mic}}

\begin{proof}
    The algorithm's correctness hinges on Step 2 correctly deciding whether to keep or remove each candidate edge $V_i \to V_j$ from the transitive closure graph $\G'$. We analyze this for true edges and non-edges separately.

    \textbf{Case 1: $V_i \to V_j$ is a true edge in $\G$.}
    The algorithm considers the edge $V_i \to V_j$ and hypothetically removes it to form $\G^-$. It then searches for a minimum vertex cut $\bvec{Z}$ in $\G^-$ that separates $V_i$ from $V_j$. Two outcomes are possible:
    \begin{enumerate}[nosep]
        \item No separating set $\bvec{Z}$ with $|\bvec{Z}| < k$ is found. The algorithm correctly keeps the edge.
        \item A separating set $\bvec{Z}$ with $|\bvec{Z}| < k$ is found. The algorithm then tests if $V_j \in \CHG(I(\{V_i\}) \mid I(\bvec{Z}))$. By definition, the set $\bvec{Z}$ blocks all paths from $V_i$ to $V_j$ in $\G$ *except for the direct edge* $V_i \to V_j$. In the context of the conditioning intervention $I(\bvec{Z})$, the only remaining path from $V_i$ to $V_j$ is the true edge itself. This makes $V_j$ a conditional MIC of $V_i$ given $I(\bvec{Z})$. By the II faithfulness assumption, $V_j$ must be in the change set. The condition $V_j \in \CHG(\dots)$ is met, and the algorithm correctly keeps the edge.
    \end{enumerate}
    In either outcome, a true edge is never removed.

    \textbf{Case 2: $V_i \to V_j$ is not a true edge in $\G$.}
    We need to show that the algorithm will find a separating set $\bvec{Z}$ with $|\bvec{Z}| < k$ and subsequently remove the edge. We prove this by induction on the position of $V_j$ in the topological order used by the algorithm. Without loss of generality, relabel the vertices $V_1, \ldots, V_n$ so that this topological order is $V_1, \ldots, V_n$.

    \textbf{Inductive Hypothesis (IH):} Assume that for every vertex $V_m$ with $m < j$, all incoming edges to $V_m$ in $\G''$ have been correctly identified; that is, the subgraph of $\G''$ induced by $\{V_1, \ldots, V_{j-1}\}$ is identical to the subgraph of $\G$ induced by the same set of vertices.

    \textbf{Inductive Step:} Consider the vertex $V_j$ and a candidate edge $V_i \to V_j$ that is in $\G'$ but not in $\G$ (so $i < j$).
    Since $V_i \to V_j$ is not a true edge, $(V_i, V_j)$ is a non-adjacent pair, and there exists a minimum vertex cut $\bvec{Z}_{\text{true}}$ in $\G$ that destroys all directed paths from $V_i$ to $V_j$. By Menger's theorem, the size of this cut equals the \emph{local} connectivity of the pair $(V_i, V_j)$---the maximum number of internally vertex-disjoint directed $V_i \rightsquigarrow V_j$ paths. This local quantity is at most $\kappa_{\max}$ by definition, and by the theorem's premise $k > \kappa_{\max}$, so $|\bvec{Z}_{\text{true}}| \le \kappa_{\max} < k$.

    Now, we must show that $\bvec{Z}_{\text{true}}$ also separates $V_i$ and $V_j$ in the current state of $\G''$ (with the edge $V_i \to V_j$ hypothetically removed). Assume for contradiction that it does not. This means there must be a path $P''$ from $V_i$ to $V_j$ in $\G''$ that avoids $\bvec{Z}_{\text{true}}$. Since $\bvec{Z}_{\text{true}}$ separates them in $\G$, this path $P''$ must contain at least one spurious edge (an edge in $\G''$ but not in $\G$). Let $(U, W)$ be the first such edge on path $P''$.

    The vertices $U$ and $W$ must both precede $V_j$ in the topological ordering. Therefore, $U, W \in \{V_1, \ldots, V_{j-1}\}$. By our Inductive Hypothesis, the subgraph of $\G''$ induced by these vertices is already identical to that of $\G$. This means the spurious edge $(U, W)$ cannot exist, a contradiction.

    Therefore, $\bvec{Z}_{\text{true}}$ is a valid vertex cut for $(V_i, V_j)$ in the current graph $\G''$. The algorithm searches for a *minimum* vertex cut, $\bvec{Z}_{\text{alg}}$, whose size must be less than or equal to the size of our valid cut: $|\bvec{Z}_{\text{alg}}| \le |\bvec{Z}_{\text{true}}| < k$.

    The condition $|\bvec{Z}_{\text{alg}}| < k$ is met. The algorithm tests if $V_j \in \CHG(I(\{V_i\}) \mid I(\bvec{Z}_{\text{alg}}))$. Since $\bvec{Z}_{\text{alg}}$ is a vertex cut separating $V_i$ and $V_j$ in the true graph $\G$, by intervention locality, $V_j$ will not be in the change set. The edge $V_i \to V_j$ is correctly removed.

    By induction, this holds for all vertices and all spurious edges are removed.
\end{proof}

\section{Path-Cancellation Comparison} \label{apx: cancellation}

Figure~\ref{fig:lots of path cancellations} repeats the cardinality comparison of Figure~\ref{fig:three algs} in the path-cancellation regime (edge weights in $[-0.1, 0.1]$), the stress test for CI-based methods.

\begin{figure}[h]
    \centering
    \includegraphics[width=0.98\linewidth]{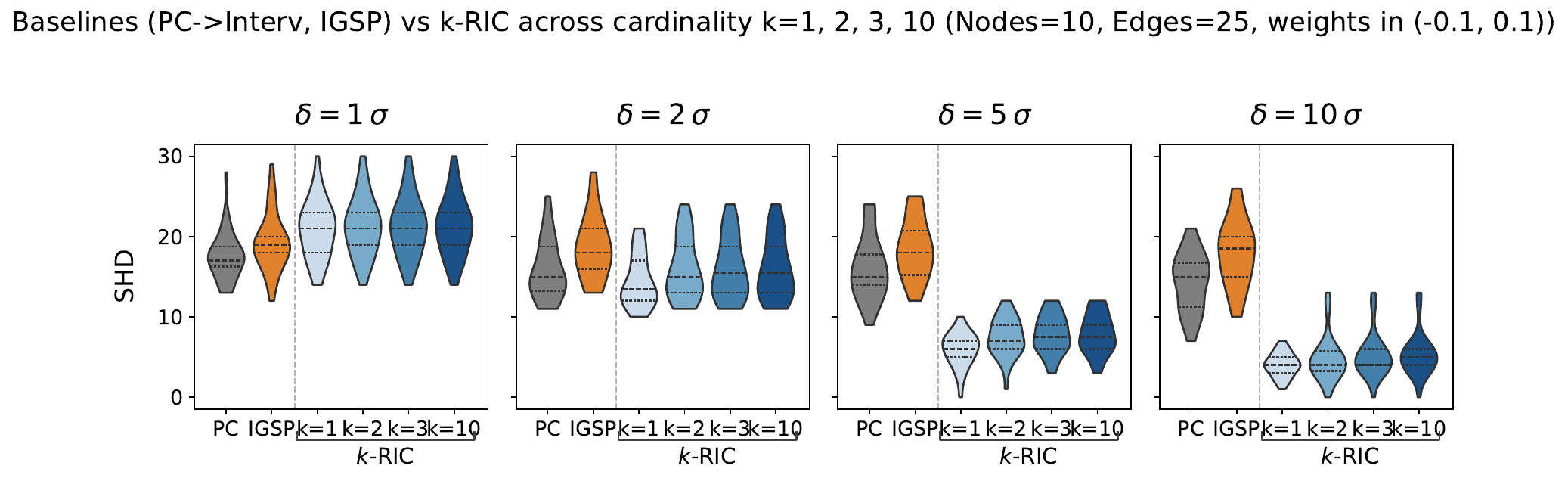}
    \caption{The cardinality comparison under significant risk of path cancellation (weights in $[-0.1, 0.1]$; $30$ DAGs, $10$ nodes, $25$ edges). The small, sign-varied coefficients make partial correlations cancel, leaving PC$\to$Interv and IGSP near-uninformative at every $\delta$, whereas intervention-only $k$-RIC, which uses no independence tests, recovers most of the structure once the intervention strength is sufficient ($\delta \gtrsim 5\sigma$).}
    \label{fig:lots of path cancellations}
\end{figure}

Small, sign-varied coefficients make partial correlations cancel, so the conditional-independence baselines PC$\to$Interv and IGSP stay near-uninformative (SHD $\approx 15$--$19$) at every $\delta$. $k$-RIC never tests independencies and is immune to this cancellation; it is limited only by whether the (now small) intervention shifts are detectable. Once $\delta \gtrsim 2\sigma$ makes them visible, even the crudest intervention-only estimate already beats both baselines, and by $\delta = 5\sigma$ $k$-RIC recovers most of the structure (SHD $\approx 6$, roughly a third of the baselines' error), reaching SHD $\approx 4$--$5$ at $\delta = 10\sigma$. Here the cardinality budget $k$ barely matters---the tiny coefficients leave the multi-node pruning tests underpowered, so extra cardinality buys little---which is exactly why we isolate the cardinality effect in the more benign standard regime of Figure~\ref{fig:three algs}, where the signal is strong enough for pruning to bite.

The $\delta$-dependence of the $k$-RIC violins in Figure~\ref{fig:lots of path cancellations} reflects the power of the underlying two-sample tests. When $\delta$ is small relative to the noise, the tests are underpowered: they miss true MIC shifts (causing missed edges) and fail to certify redundant edges as prunable (causing under-pruning), which both raise and widen the SHD distribution and account for the residual error even at large $k$. As $\delta$ grows, both failure modes recede and the distributions concentrate at low SHD. This is precisely the $\lambda$-strong mean II-faithfulness regime of Section~\ref{sec: parametric ii faith}: the detectable mean shift at a MIC scales with the product of the intervention strength $\delta$ and the edge coefficient, so shrinking $\delta$ shrinks the effective $\lambda$ that is satisfied (Figure~\ref{fig:heatmap}).

\section{Coefficient-Strength Dependence of the Baselines} \label{apx: omega lineplot}

Figure~\ref{fig:heatmap} reports mean SHD across the $(\omega, \delta)$ plane; Figure~\ref{fig:omega lineplot} isolates the coefficient-strength dependence by fixing $\delta = 5\sigma$ and plotting SHD against $\omega$ for the three methods. PC$\to$Interv is nearly flat in $\omega$: it recovers the skeleton from observational conditional independence but orients every edge directly from the interventional shifts, which stay strong at $\delta = 5\sigma$ regardless of $\omega$. IGSP, whose permutation search relies on observational CI tests throughout, degrades steadily as $\omega$ shrinks and those tests lose power. The intervention-first pipeline is lowest across the range. This is the same mechanism, now at the level of whole algorithms, that motivates preferring interventions to conditional independence: the more a method's decisions rest on CI tests, the more it suffers as the parametric signal weakens.

\begin{figure}[h]
    \centering
    \includegraphics[width=0.68\linewidth]{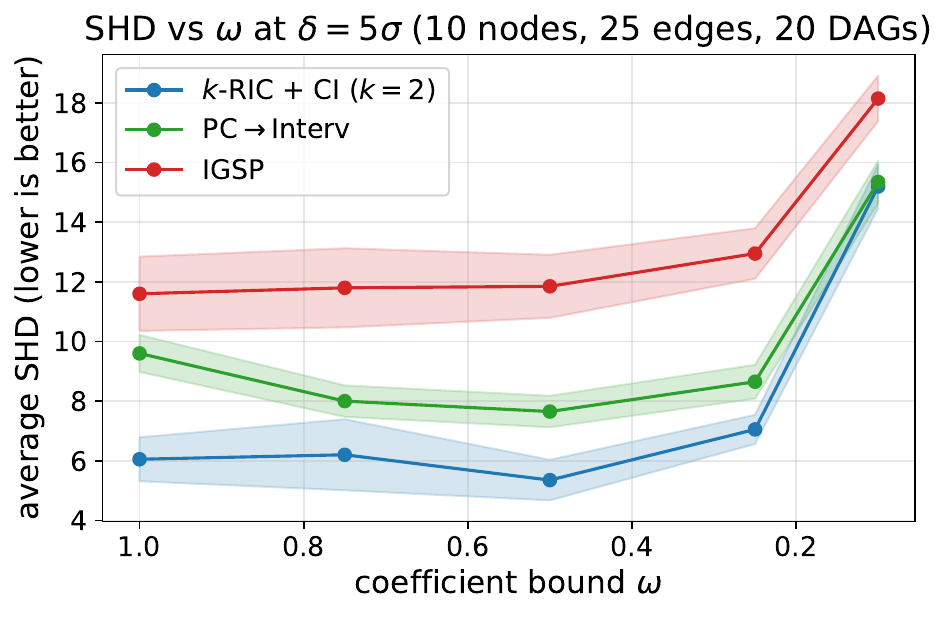}
    \caption{Average SHD versus the coefficient bound $\omega$ at fixed intervention strength $\delta = 5\sigma$ ($10$-node graphs, $25$ edges; shaded bands are $\pm 1$ standard error). PC$\to$Interv stays roughly flat because it orients edges from the interventions; IGSP degrades as $\omega$ weakens because its permutation search depends on observational CI tests; the intervention-first pipeline ($k$-RIC $+$ CI, $k=2$) is lowest throughout.}
    \label{fig:omega lineplot}
\end{figure}

\end{document}